\documentclass{article}

\usepackage[preprint]{neurips_2026}

\usepackage[utf8]{inputenc}
\usepackage[T1]{fontenc}
\usepackage{hyperref}
\usepackage{url}
\usepackage{booktabs}
\usepackage{amsfonts}
\usepackage{nicefrac}
\usepackage{microtype}
\usepackage{xcolor}

\usepackage{graphicx}
\usepackage{subcaption}
\usepackage{dblfloatfix}

\usepackage{amsmath}
\usepackage{amssymb}
\usepackage{mathtools}
\usepackage{amsthm}

\usepackage{algorithm}
\usepackage{algorithmic}

\usepackage[capitalize,noabbrev]{cleveref}

\theoremstyle{plain}
\newtheorem{theorem}{Theorem}[section]

\newtheorem{corollary}[theorem]{Corollary}

\theoremstyle{definition}

\theoremstyle{remark}

\title{Discrete Diffusion with Sample-Efficient Estimators for Conditionals}

\author{
  Karthik Elamvazhuthi \\
  \And
  Abhijith Jayakumar \\
  \And
  Andrey Y. Lokhov
}

\begin{document}

\maketitle


\begin{abstract}
We study a discrete denoising diffusion framework that integrates a sample-efficient estimator of single-site conditionals with round-robin noising and denoising dynamics for generative modeling over discrete state spaces. Rather than approximating a discrete analog of a score function, our formulation treats single-site conditional probabilities as the fundamental objects that parameterize the reverse diffusion process. We employ a sample-efficient method known as Neural Interaction Screening Estimator (NeurISE) to estimate these conditionals in the diffusion dynamics. Controlled experiments on synthetic Ising models, MNIST, and scientific data sets produced by a D-Wave quantum annealer, synthetic Potts model and one-dimensional quantum systems demonstrate the proposed approach. On the binary data sets, these experiments demonstrate that the proposed approach outperforms popular existing methods including ratio-based approaches, achieving improved performance in total variation, cross-correlations, and kernel density estimation metrics.
\end{abstract}
\section{Introduction}
Generative modeling over discrete spaces is fundamental to a wide range of applications, including molecular design, language modeling, and policy learning in reinforcement learning \cite{ho2016generative, bengio2003neural, jin2018junction}. In these settings, data consist of categorical or binary variables with complex statistical dependencies, and accurately capturing their joint structure requires models that can scale to high-dimensional combinatorial configuration spaces.

While diffusion models have revolutionized generative modeling in continuous domains \cite{ho2020denoising}, their direct application to discrete data has received increased attention. Continuous-time formulations rely on Gaussian noise and score estimation through gradients of log-densities, quantities that are not well defined in discrete spaces. Naive relaxations, such as adding continuous noise to one-hot encodings, break the discrete structure and often yield poor sample quality or unstable training.  

These limitations motivate the need for a principled framework for \emph{discrete diffusion processes} that preserves the combinatorial structure of the data, allows tractable inference, and retains the interpretability and scalability that made diffusion models successful in continuous domains.

A number of works have considered extending diffusion models to discrete spaces. \cite{austin2021structured} considers denoising diffusion for discrete data and discrete time (both absorbing and uniform diffusion), optimizing the variational lower bound (VLB) of the log-likelihood.  
A continuous-time framework for discrete diffusion models is introduced in \cite{campbell2022continuous}, which also optimizes the VLB of the log-likelihood.  The work
\cite{sun2022score} performs score matching for continuous-time diffusion by learning conditionals using cross-entropy.  
In \cite{lou2024discrete}, the authors propose an approach to learn the discrete version of the score using a score-entropy function to ensure non-negativity of the score along the training iterations. For a broader survey on discrete diffusion models see \cite{li2025survey}.


This paper’s contribution begins with making explicit that for forward transitions the canonical time-reversed kernel can be parameterized entirely through ratios of probabilities between configurations that differ at one coordinate, and that these ratios reduce exactly to ratios of single-site conditional distributions given the remaining coordinates. So reverse diffusion can be implemented by learning local conditionals instead of a global density or discrete {\it score.} 
The main methodological contribution following this observation is to plug in
state-of-the-art estimator for learning discrete conditionals. Motivated by theoretical results on the sample-efficiency of interaction-screening estimators, including RISE \cite{vuffray2016interaction} and GRISE \cite{jayakumar2020learning} for learning discrete conditional distributions, we use the Neural Interaction Screening Estimator (NeurISE) \cite{jayakumar2020learning} that has shown strong empirical advantage for learning discrete conditionals with unknown structure. We apply this estimator to the round-robin noising diffusion scheme previously also explored in \cite{varma2024glauber}.


We evaluate the proposed estimation approach for ratios of conditionals using NeuRISE on a range of discrete generative modeling
benchmarks, including synthetic Ising on a $25$-variable system, binarized
MNIST, and \textcolor{black}{quantum annealing (D-Wave) datasets}. Across these settings, we compare
against the classifier based estimator of \cite{varma2024glauber} and three other methods: ELBO-based method D3PM \cite{austin2021structured}, the score based approach of SEDD method \cite{lou2024discrete}, and the flow matching approach of discrete flow matching \cite{gat2024discrete}. We demonstrate consistent improvements in distributional accuracy with comparable run times with existing models.  A limitation of our study is that we didn't train on very large architectures or data sets. Instead, our experiments are intentionally designed around controlled smaller-scale benchmarks. This choice allows us to study the statistical behavior of the proposed estimator and compare denoising algorithms under repeated trials, varying sample sizes, and measurable error decay. In particular, the $25$-variable Ising model serves as a key controlled setting in which exact samples can be produced and total variation distance can be computed exactly, enabling a detailed comparison of how different methods improve with additional training data.

An important observation we make is that, under the round-robin noising scheme, autoregressive generation arises as the hard-noise limit of the reverse-time sampler. In this limit, each reverse step resamples one coordinate from its single-site conditional distribution, and applying these updates in a fixed order recovers an autoregressive sampling procedure. This provides a unified view of the diffusion sampler and the autoregressive sampler associated with \cite{varma2024glauber}, without introducing a separate autoregressive model. This connection is distinct from that of \cite{ou2024your}, who relate absorbing diffusion to any-order autoregressive models at the level of the training objective. In our setting, the connection occurs at the sampler level. That is, under round-robin noising, autoregressive generation emerges directly from the finite-step hard-noise reverse process.

\section{Problem}
Let \( \Sigma \) denote a discrete set of alphabets with cardinality $|\Sigma| = p$. We define the configuration space  $\Sigma^q $ for $ q $ discrete variables or coordinates, with elements denoted as $\sigma := (\sigma_1,...\sigma_q)$. Given training samples from a probability distribution $\mu:  \Sigma^q  \rightarrow \mathbb{R}$, \emph{the goal is to construct a diffusion-based generative model from which new samples can be tractably generated}.

We will specifically focus on a general class of probabilistic models over \( \Sigma^q \) defined using a Hamiltonian \( H : \Sigma^q \to \mathbb{R} \) : $
\mu(\sigma) = \frac{1}{Z} \exp\left(H(\sigma)\right), \quad \sigma \in \Sigma^q,$
where the partition function \( Z = \sum_{\sigma \in \Sigma^q} \exp\left(H(\sigma)\right) \) ensures normalization. This formulation defines an energy-based probability distribution from the exponential family, where the energy function \( H(\sigma) \) typically encodes interactions between components of \( \sigma \), such as pairwise terms or external fields.

An important example of such a model is the \emph{Ising model}, where $\Sigma : = \{-1,1 \}$which is a pairwise Markov random field over \( \Sigma^q \) with interactions defined by an undirected graph \( G = (V, E) \), where \( V = [q] \) and \( E \subseteq V \times V \). The Hamiltonian takes the form:
$
H(\sigma) = \sum_{(i,j) \in E} J_{ij} \sigma_i \sigma_j + \sum_{i \in [q]} h_i \sigma_i. $
Here \( J_{ij} \in \mathbb{R} \) represents the strength of the interaction between nodes \( i \) and \( j \),
 \( h_i \in \mathbb{R} \) represents an external bias or field at node \( i \),
 \( \sigma_i \in \Sigma \) for all \( i \in [q] \).
This model defines a probability distribution:
\begin{equation}
\mu(\sigma) = \frac{1}{Z} \exp\left( \sum_{(i,j) \in E} J_{ij} \sigma_i \sigma_j + \sum_{i \in [q]} h_i \sigma_i \right), \quad \sigma \in \Sigma^q.
\label{eq:ising}
\end{equation}


\section{Discrete Diffusion through Conditionals }

We consider a denoising diffusion framework over the configuration space \( \Sigma^q \), which consists of a known forward Markov process and a learned reverse process. The goal is to construct a forward Markov process \( \{X_n\}_{n=0}^T \) such that $X_0 \sim \mu_0$ is the data distribution that we are interested in learning. The forward process then makes sure that the distribution \( \mu_n \) of \( X_n \) converges to a tractable noise distribution, from which it is easy to sample: $\lim_{n \rightarrow \infty} \mu_n = \mu_{\text{noise}}$.
The Markov chain over \( \Sigma^q \) evolves according to a known transition kernel
$k_n : \Sigma^q \times \Sigma^q \rightarrow \mathbb{R}_{\geq 0}$  which defines the conditionals $k_n(\sigma, \tilde{\sigma}):=\mathbb{P}(X_{n+1} = \sigma | X_n = \tilde{\sigma})$.
The forward evolution of the distribution $\mu_n$ of the process $X_n$ is then given by:
\begin{equation}
\mu_{n+1}(\sigma) = \sum_{\tilde{\sigma} \in \Sigma^q} k_n(\sigma, \tilde{\sigma}) \mu_n(\tilde{\sigma}). 
\end{equation}
\subsection*{Reverse Process}

To sample from the target distribution \( \mu_0 \), we define a reverse process \( \{Y_n\}_{n=0}^T \) such that \( Y_n \sim \mu_{T-n} \). This reverse process is governed by a family of time-inhomogeneous transition kernels $
k_n^{\text{rev}} : \Sigma^q \times \Sigma^q \rightarrow \mathbb{R}_{\geq 0},$
which satisfy the backward recurrence:
\begin{equation}
\label{eq:revp}
\mu_n(\sigma) = \sum_{\tilde{\sigma} \in \Sigma^q} k_n^{\text{rev}}(\sigma, \tilde{\sigma}) \mu_{n+1}(\tilde{\sigma}).
\end{equation}
We obtain a natural candidate for the reverse-time transition kernel via Bayes’ rule is the following:
\begin{equation} 
 k_n^{\text{rev}}(\sigma, \tilde{\sigma}) = k_n(\tilde{\sigma}, \sigma) \cdot \frac{\mu_n(\sigma)}{\sum_{\hat{\sigma} \in \Sigma^q} k_n(\tilde{\sigma}, \hat{\sigma}) \mu_n(\hat{\sigma})}  = \frac{k_n(\tilde{\sigma}, \sigma)}{\sum_{\hat{\sigma} \in \Sigma^q} k_n(\tilde{\sigma}, \hat{\sigma}) \cdot \frac{\mu_n(\hat{\sigma})}{\mu_n(\sigma)}}.
 \label{eq:revkern}
 \end{equation}

This expression shows that the reverse kernel depends only on the forward transition
probabilities and on ratios of the form
\(\mu_n(\hat{\sigma})/\mu_n(\sigma)\). Consequently, if these ratios can be accurately
estimated we can construct accurate approximations to the reverse
kernels, \(k_n^{\mathrm{rev}}\). We can then sample from the noise distribution
\(\mu_{\mathrm{noise}}\) and iteratively apply the reverse transitions to generate new samples
from an approximation of the data distribution. 

This captures the core idea of denoising diffusion probabilistic models (DDPMs). The
forward process gradually drives the data distribution toward a simple noise
distribution, while the reverse process reconstructs samples by inverting this dynamics
using learned conditional structure. Theorem \ref{thm:errorbndmain} in the Appendix formalizes this intuition as discrete-time discrete-state analogue of existing results in literature for continuous-time  continuous space \cite{de2021diffusionSB,chen2022sampling,chen2024convergence}, and continuous-time discrete state \cite{zhang2024convergence,ren2024discrete}, added for the sake of completeness.


\subsection*{Round-Robin Forward Noising}

We now describe a choice of forward dynamics, according to a noising scheme introduced in \cite{varma2024glauber}, that will be used in this paper. In this choice of the forward process, we gradually introduce noise into the configuration by modifying one coordinate (e.g., a pixel or spin) at a time. An advantage of this form of noising is that the number of ratios that are required to be learned for each time step is much smaller than with other schemes \cite{austin2021structured,sun2022score,lou2024discrete}, where all the variables are noised simultaneously. The precise scheme for noising is the following:
\begin{enumerate}
\item  A randomization parameter \( \varepsilon \in [0,1] \) is fixed. 
\item At each time step \( n \in \{1, \dots, T\} \), a specific coordinate is selected in round-robin order: the \( n \)-th coordinate of $\sigma$ is selected as $u =((n-1) \mod q)+1$. 
\item With probability $\varepsilon$ the coordinate $\sigma_u$ is left unchanged. With probability $1-\varepsilon$ the coordinate value $\sigma_u$ is uniformly randomly sampled from $\Sigma$. 
\end{enumerate}
Since there are $p$ number of elements in $\Sigma$, the conditional probabilities if this forward noising process are given by:
\begin{align*}
   k_n(\sigma , \tilde{\sigma} ) = \begin{cases}
\frac{1 - \varepsilon}{p}, & \text{if } \sigma_{-u} = \tilde{\sigma}_{-u},~ \sigma_u \neq \tilde{\sigma}_u, \\
\frac{1 -\varepsilon}{p} + \varepsilon, & \text{if } \sigma = \tilde{\sigma}, \\
0, & \text{otherwise},
\end{cases}
\end{align*}
where $\sigma_{-u} \in \Sigma^{q-1}$ denotes the configuration excluding the
$u$-th coordinate. 
For notational convinience in the forthcoming expressions, we define the parameters,
$
a= \frac{1-\varepsilon}{p},~~ b = \frac{1 -\varepsilon}{p} + \varepsilon,
$
the probability of noising the chosen coordinate to an alphabet different from its current value, and the probability of it picking the current alphabet again, respectively. 

Since the only admissible transitions are one coordinate transition away at the noised coordinate $u=((n-1) \mod q)+1$, we can substitute the expression for $k_n(\sigma, \tilde{\sigma})$ in \eqref{eq:revkern} to express the reverse kernel as,
\begin{align*}
k_n^{\mathrm{rev}}(\sigma,\tilde{\sigma})
&=
\frac{a\,\mu_n(\sigma)}{
a\,\mu_n(\sigma)
+ b\,\mu_n(\tilde{\sigma})
+ a \!\!\!\sum\limits_{\hat{\sigma}\in\mathcal{N}_u(\tilde{\sigma})\setminus\{\sigma,\tilde{\sigma}\}}
\!\!\!\mu_n(\hat{\sigma})
} 
\end{align*}
with $
\mathcal{N}_u(\tilde{\sigma})
:= \{\hat{\sigma} \in \Sigma^q : \hat{\sigma}_{-u} = \tilde{\sigma}_{-u}\}.$
Therefore, this gives rise to three possibilities,
\begin{align}
\label{eq:k_rev1}
k_n^{\mathrm{rev}}(\sigma,\tilde{\sigma})
&=
\frac{
   a\,\mu_n(\sigma)
}{
   b\,\mu_n(\tilde{\sigma})
   + a\sum\limits_{\mathcal{N}_u(\tilde{\sigma})\setminus\{\tilde{\sigma}\}}
       \mu_n(\hat{\sigma})
},
\end{align}
when $\sigma_{-u} = \tilde{\sigma}_{-u}, \ \text{but} \ \ \sigma \neq \tilde{\sigma}$. 
Similarly, 
\begin{align}
\label{eq:k_rev2}
k_n^{\mathrm{rev}}(\sigma,\sigma)
=
\frac{
   b\,\mu_n(\sigma)
}{
   b\,\mu_n(\sigma)
   + a\sum\limits_{\mathcal{N}_u(\sigma)\setminus\{\sigma\}}\mu_n(\hat{\sigma})
},
\end{align}
and $k_n(\sigma, \tilde{\sigma}) = 0$ otherwise.
Of special interest is the case when the discrete set is binary: $\Sigma := \{ -1,1\}$ is binary. Then, we get,

\begin{equation}
k_n(\sigma, \tilde{\sigma}) =
\begin{cases}
\frac{1 - \varepsilon}{2}, & \text{if } \sigma_{-u} = \tilde{\sigma}_{-u},~ \sigma_u \neq \tilde{\sigma}_u, \\
\frac{1 + \varepsilon}{2}, & \text{if } \sigma = \tilde{\sigma}, \\
0, & \text{otherwise},
\end{cases}.
\end{equation}

The corresponding reverse transition probabilities are:
\begin{align*}
k_n^{\text{rev}}(\sigma, \tilde{\sigma}) &= (1 - \varepsilon) \cdot \frac{\mu_n(\sigma)}{(1 - \varepsilon)\mu_n(\sigma) + (1 + \varepsilon)\mu_n(\tilde{\sigma})}, \\
& \quad \text{if } \sigma_{-u} = \tilde{\sigma}_{-u} \ \text{and} \ ~ \sigma_u \neq \tilde{\sigma}_u,
\end{align*}
and
\begin{align*}
k_n^{\text{rev}}(\sigma, \sigma) &= (1 + \varepsilon) \cdot \frac{\mu_n(\sigma)}{(1 + \varepsilon)\mu_n(\sigma) + (1 - \varepsilon)\mu_n(\tilde{\sigma})}. \\
& \quad 
\end{align*}

\textbf{The Hard Noise Autoregressive Limit}

We now consider a special case of the forward process in which noise is {\it harsh}:   $\varepsilon = 0$. This corresponds to
a full randomization of the selected coordinate at each step, and hence all the information is lost in the variable after the corresponding noising step. Our goal in this section is to show that we recover auto-regressive generation in this limit.

Let $T = q$, and define a time-inhomogeneous Markov kernel that updates only the
$n$-th coordinate at time step $n \in \{1, \ldots, T\}$.
The forward transition kernel simplifies to,
\begin{equation}
k_n(\sigma, \tilde{\sigma}) =
\begin{cases}
\frac{1}{p}, & \text{if } \sigma_{-n} = \tilde{\sigma}_{-n}, \\[6pt]
0, & \text{otherwise},
\end{cases}
\end{equation}

Over $T = p$ steps, this procedure fully randomizes each coordinate once,
resulting in convergence to the uniform distribution over $\Sigma^q$.

By the reverse kernel construction, we obtain
\begin{align*}
 \mathbb{P}(X_n = \sigma \mid X_{n+1} = \tilde{\sigma})
= k_n^{\mathrm{rev}}(\sigma, \tilde{\sigma}) 
= \mathbf{1}_{\sigma_{-n} = \tilde{\sigma}_{-n}}
\cdot
\frac{\mu_n(\sigma)}
{\sum_{ \hat\sigma \in \mathcal{N}_n(\sigma)} \mu_n(\hat{\sigma})}.
\end{align*}

This corresponds to sampling the $n$-th coordinate conditioned on the others:
\[
\mathbb{P}(X_n^n = \sigma_n \mid X_n^{-n} = \tilde{\sigma}_{-n})
=
\frac{\mu_n(\sigma)}
{\sum_ {\hat\sigma \in \mathcal N_n(\sigma)} \mu_n(\hat{\sigma})},
\]
where the sum is taken over all configurations
$\hat{\sigma} \in \Sigma^q$ that agree with $\sigma$ outside coordinate $n$.

We now express this reverse process recursively.
Let $X_T = \tilde{\sigma}$ be a configuration sampled from the noise distribution.
At each reverse step $n$, the process samples a configuration $X_n = \sigma$
such that $\sigma_{-n} = \tilde{\sigma}_{-n}$, while drawing $\sigma_n$
from the corresponding conditional distribution.

Formally,
\begin{align*}
 \mathbb{P}(X_n = \sigma \mid X_{n+1} = \tilde{\sigma}) 
=
\mathbf{1}_{\sigma_{-n} = \tilde{\sigma}_{-n}}
\cdot
\mathbb{P}(X_n^n = \sigma_n \mid X_n^{-n} = \tilde{\sigma}_{-n}).
\end{align*}

Unrolling the reverse chain from $T$ to $0$ yields
\begin{align}
\label{eq:rev-unroll}
 \mathbb{P}(X_0 = \sigma \mid X_T = \tilde{\sigma})
 =
\prod_{n=1}^{T}
\mathbb{P}(X_{T-n} = \sigma_{T-n}
\mid
X_{T-n+1} = \tilde{\sigma}_{T-n+1}).
\end{align}

Since the two configurations differ only at coordinate $n$,
\begin{equation}
\label{eq:single-step}
\mathbb{P}(X_{T-n} = \sigma
\mid
X_{T-n+1} = \tilde{\sigma})
=
\mathbf{1}_{\sigma_{-n} = \tilde{\sigma}_{-n}}
\,
\mathbb{P}(\sigma_n \mid \tilde{\sigma}_{-n}).\nonumber
\end{equation}

Substituting into~\eqref{eq:rev-unroll}, we obtain
\begin{equation}
\label{eq:prod-simplified}
\mathbb{P}(X_0 = \sigma \mid X_T = \tilde{\sigma})
=
\prod_{n=1}^{T}
\mathbf{1}_{\sigma_{-n} = \tilde{\sigma}_{-n}}
\,
\mathbb{P}(\sigma_n \mid \tilde{\sigma}_{-n}). \nonumber
\end{equation}

If the reverse update keeps non-updated coordinates fixed, i.e.\
$X_n^{-n} = X_{n+1}^{-n}$, then the conditioning simplifies to
\begin{equation}
\label{eq:autoregressive}
\mathbb{P}(X_0 = \sigma \mid X_T = \tilde{\sigma})
=
\prod_{n=1}^{T}
\mathbb{P}(\sigma_n \mid \tilde{\sigma}_{>n}),
\end{equation}
which recovers an autoregressive factorization over the discrete alphabet
$\Sigma$.

\section{Learning Conditionals using Neural Interaction Screening}
Let $\mu_n$ denote the probability distribution of the random variable $X_n$.
In order to implement the reverse dynamics, we need to estimate the ratio $
\frac{\mu_n(\tilde{\sigma})}{\mu_n(\sigma)}$
from samples of the forward process. Suppose \( \sigma, \tilde{\sigma} \in \Sigma^q \) differ at only one coordinate \( n \), i.e.,
$
\tilde{\sigma}_i = \sigma_i \quad \text{for all } i \neq n, ~\tilde{\sigma}_n \neq \sigma_n.
$

Then, for any distribution \( \mu_n \) over \( \Sigma^q \), we have:
\[
\frac{\mu_n(\tilde{\sigma})}{\mu_n(\sigma)} = \frac{\mu_n(\tilde{\sigma}_n \mid \sigma_{-n})}{\mu_n(\sigma_n \mid \sigma_{-n})},
\]
where \( \sigma_{-n} \in \Sigma^{q-1} \) denotes the shared values of the configuration outside the \( n \)-th coordinate.

\medskip

\noindent
This identity follows directly from the definition of conditional probability
$
\mu_n(\sigma) = \mu_n(\sigma_n \mid \sigma_{-n}) \cdot \mu_n(\sigma_{-n})
$.
Taking the ratio, we obtain:
\[
\frac{\mu_n(\tilde{\sigma})}{\mu_n(\sigma)} 
= \frac{\mu_n(\tilde{\sigma}_n \mid \sigma_{-n}) \cdot \mu_n(\sigma_{-n})}
       {\mu_n(\sigma_n \mid \sigma_{-n}) \cdot \mu_n(\sigma_{-n})}
= \frac{\mu_n(\tilde{\sigma}_n \mid \sigma_{-n})}{\mu_n(\sigma_n \mid \sigma_{-n})}.
\]

\medskip

\noindent
This expression provides a tractable way to compute (or approximate) the required ratio using only local conditionals for the revese dynamics.

To compute the required single-site conditional distributions, we use the
\emph{Neural Interaction Screening Estimator} (NeurISE) \cite{jayakumar2020learning},
which learns local conditionals in discrete graphical models by neural
parameterization of partial energy functions. 

This local conditional modeling is well matched to the reverse diffusion kernel,
which depends only on ratios of single-site conditionals between configurations
differing at one coordinate, enabling efficient and scalable implementation of
the reverse-time dynamics without explicitly modeling the global distribution.

Following NeurISE \cite{jayakumar2020learning}, we introduce the centered indicator embedding
\[
\Phi_s(r) :=
\begin{cases}
1 - \dfrac{1}{q}, & r = s, \\[6pt]
-\dfrac{1}{q}, & r \neq s,
\end{cases}
\qquad s,r \in \Sigma .
\]
We then define the vector-valued embedding, $
\Phi(r) := \bigl(\Phi_1(r), \ldots, \Phi_q(r)\bigr) \in \mathbb{R}^q,
$

Suppose that 
$\mu_n(\sigma) \propto \exp(H(\sigma))$ is a Gibbs distribution for some Hamiltonian function 
$H : \Sigma^{q} \to \mathbb{R}$.
For any coordinate \(u \in [1,,,q]\), there always exists a decomposition $
H(\sigma)
=
H_{- u}(\sigma_{- u})
+
H_u(\sigma),
$
where \(H_{- u}\) does not depend on \(\sigma_u\).w
Substituting it into the Gibbs distribution
$\mu_n(\sigma) \propto \exp(H(\sigma))$ yields
\begin{align*}
 \mu_n(\sigma_u \mid \sigma_{-u}) 
=
\frac{\exp\!\left(H_{- u}(\sigma_{-u}) + H_u(\sigma)\right)}
{\sum_{r \in \Sigma}
\exp\!\left(H_{- u}(\sigma_{-u}) + H_u(\sigma_u = r, \sigma_{-u})\right)}.
\end{align*}
Since $H_{- u}(\sigma_{-u})$ does not depend on $\sigma_u$, we get
\[
\mu_n(\sigma_u \mid \sigma_{-u})
=
\frac{\exp\!\left(H_u(\sigma)\right)}
{\sum_{r \in \Sigma}
\exp\!\left(H_u(\sigma_u = r, \sigma_{-u})\right)}.
\]
Therefore, the partial energy $H_u$ determines the single-site conditional
distribution up to an additive function of $\sigma_{-u}$, and can be written as
$
H_u(\sigma)
=
\log \mu_n(\sigma_u \mid \sigma_{-u})
+
\mathrm{const}(\sigma_{-u}).
$
For each $\sigma_{-u}$, $H_u(\cdot,\sigma_{-u})$ can chosen be to satisfy $
\sum_{r \in \Sigma} H_u(r, \sigma_{-u}) = 0$.
The functions $\Phi_r$ form a basis for functions, $f:\Sigma \rightarrow \mathbb{R}$ that average to $0$ or are {\it centered}.
Therefore, for each coordinate $u \in [q]$, we approximate this partial energy using a
neural network $
\mathrm{NN}_\theta : \Sigma^{q-1} \to \mathbb{R}^q$
Specifically, we use this parameterization at the Hamiltonian level as follows,
\begin{align}
\widetilde H_u(\sigma; w)
&=
\bigl\langle \Phi(\sigma_u), \mathrm{NN}_\theta(\sigma_{-u}) \bigr\rangle 
  = 
\sum_{s=1}^q
\Phi_s(\sigma_u)\,
\mathrm{NN}_\theta(\sigma_{-u})_s. \nonumber
\end{align}
This representation is fully general without loss of expressivity.
Given samples \(\{\sigma^{(n)}\}_{n=1}^N\) from the forward process at time \(n\), the \emph{NeurISE loss} as presented in \cite{jayakumar2020learning}, 
for site \(u\) is
\[
\mathcal{L}_u(\theta)
=
\frac{1}{N}
\sum_{n=1}^N
\exp\!\left(
-
\bigl\langle
\Phi(\sigma_u^{(n)}),
\mathrm{NN}_\theta(\sigma_{-u}^{(n)})
\bigr\rangle
\right).
\]

Here $\theta$ represents the trainable parameters of the neural net

Since we have conditionals that need to be learned be for each time step, we introduce a neural network $NN_{\theta} : \mathbb{R} \times  \mathbb{R}^q \times  \mathbb{R}^{q-1} \rightarrow \mathbb{R}^p$ that accepts arguments $(t,u,\sigma_{-u})$ where the coordinate $u$ is encoded as a one-hot vector. This gives us the composite loss
\begin{align}
\begin{aligned}
\mathcal{L}_u(\theta) = 
\frac{1}{TN}\sum_{s = 1}^T 
\sum_{n=1}^N
\exp\!\left(
-
\bigl\langle
\Phi((X^n_s)_u),
\mathrm{NN}_\theta(t,u,(X^n_s)_{-u}
\bigr\rangle
\right).
\end{aligned}
\label{eq:NISE}
\end{align}
\paragraph{Learned Conditional Distribution.}

Once trained, the approximate conditional distribution is recovered by setting
\begin{equation}
\widehat{\mu}_n(\sigma_u \mid \sigma_{-u})
=
\frac{
\exp\!\left(
\bigl\langle \Phi(\sigma_u), \mathrm{NN}_u(\sigma_{-u}) \bigr\rangle
\right)
}{
\sum_{r \in \Sigma}
\exp\!\left(
\bigl\langle \Phi(r), \mathrm{NN}_u(\sigma_{-u}) \bigr\rangle
\right)
}. 
\label{eq:condexp}
\end{equation}
Therefore, for any pair \((\sigma, \tilde\sigma)\) differing at coordinate \(u\), the ratio required
for the reverse diffusion kernel is given by
\[
\frac{\mu_n(\tilde\sigma)}{\mu_n(\sigma)}
\;\approx\;
\frac{
\exp\!\left(
\bigl\langle \Phi(\tilde\sigma_u), \mathrm{NN}_u(\sigma_{-u}) \bigr\rangle
\right)
}{
\exp\!\left(
\bigl\langle \Phi(\sigma_u), \mathrm{NN}_u(\sigma_{-u}) \bigr\rangle
\right)
}.
\]

\section{Numerical Experiments}

\begin{figure*}[t]
    \centering
    \begin{subfigure}[b]{0.32\textwidth}
        \centering
        \includegraphics[width=1.15\textwidth]{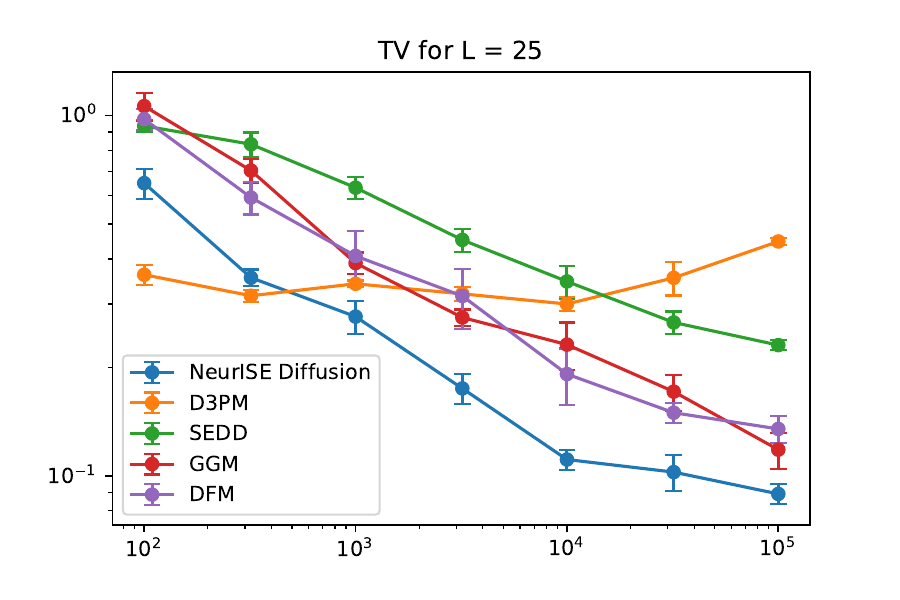}
        \caption{TV}
        \label{fig:subfig1}
    \end{subfigure}
    \hfill
    \begin{subfigure}[b]{0.32\textwidth}
        \centering
        \includegraphics[width=1.15\textwidth]{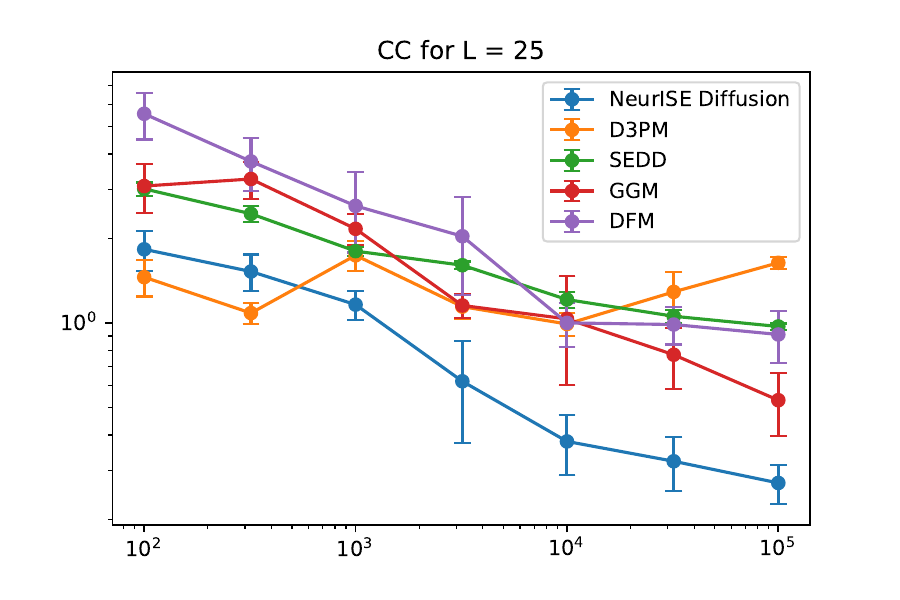}
        \caption{Cross-correlation Error}
        \label{fig:subfig2}
    \end{subfigure}
   \hfill
    \begin{subfigure}[b]{0.32\textwidth}
        \centering
        \includegraphics[width=1.15\textwidth]{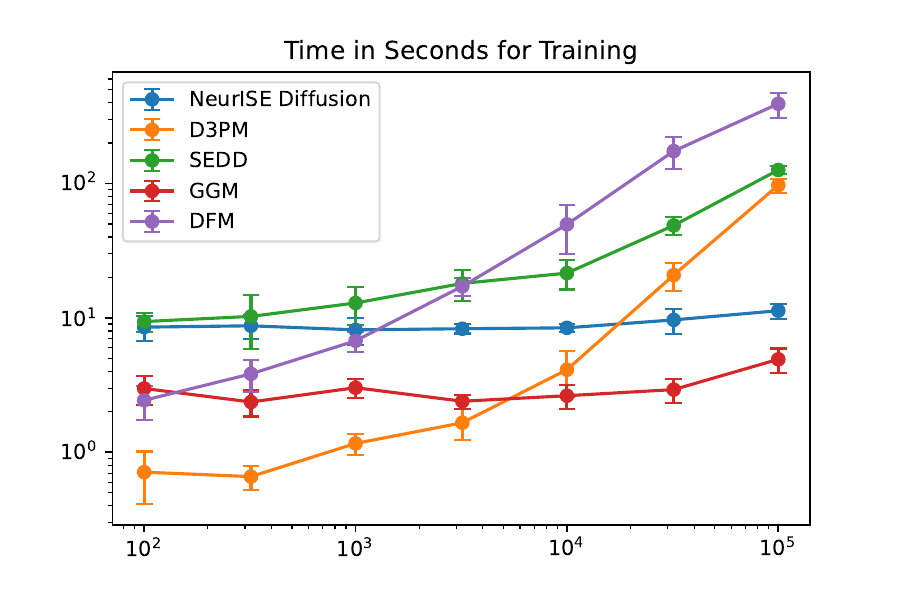}
        \caption{Training Time}
        \label{fig:subfig3}
    \end{subfigure}

    \caption{Trend of TV and Cross-correlation error as a function of training set size for Ising models, averaged over $5$ Ising models and $10$ trials per data set.}
    \label{fig:combined}
\end{figure*}
 
In this section, we compare our denoising method that we refer to as \emph{NeurISE diffusion}, with the algorithm proposed in \cite{varma2024glauber} (GGM), ELBO-based method of D3PM \cite{austin2021structured},  the score matching approach (SEDD) proposed in \cite{lou2024discrete} and discrete flow matching (DFM) \cite{gat2024discrete}. The relevant code used to run the experiments can be found in our \href{https://github.com/lanl-ansi/NeurISEdiffusion}{Github repository}. We implemented our own version of GGM and SEDD, the D3PM implementation was adapted from an unofficial publicly available implementation~\cite{d3pm_pytorch}, and DFM was adapted from the publicly provide code in the guide \cite{lipman2024flowmatchingguidecode}. The core parametric model we will use in our method and each of these methods will be multilayer perceptrons with different depth, depending on the test cases, with batch normalization layers. We make this choice to study the performance of these methods on an equal footing, decoupled from representational differences coming from the model.

\subsection{Test Case 1: Binary Data}
\label{ssec:EAm}

\paragraph{Ising Model} We first compare different methods on small-scale synthetic data. The benchmark is based on the \emph{Edwards-Anderson (EA) model} \cite{edwards1975theory, bhatt1988numerical}, which is a specific instance of the binary Ising model defined by the distribution in Equation~\eqref{eq:ising}. The EA model Hamiltonian is defined over \( p = L^2 \) binary variables \( \sigma_i \in \{-1, +1\} \), arranged on a two-dimensional square lattice of size \( L \times L \). 
The graph \( E \subseteq [q] \times [q] \) corresponds to the set of nearest-neighbor pairs on a 2D periodic grid. Each spin \( \sigma_i \) interacts with its right and bottom neighbors, with periodic boundary conditions applied in both directions. The pairwise couplings \( J_{ij} \) are symmetric random variables sampled independently for each edge \( (i,j) \in E \) as:
  $
  J_{ij} = J_{ji} \in \{-1.2, +1.2\}.
  $
The local fields \( h_i \in \{-0.05, +0.05\} \) are also sampled independently for each node \( i \in [q] \).

In our experiments, we use a lattice size of \( L = 5 \), resulting in a model with \( q = 25 \) binary variables. The results of using a two layer MLP can be seen for different diffusion methods in \figureautorefname ~\ref{fig:combined}. We test the models for different values of training data across averaged over $5$ different choices of EA models, with $10$ runs for differing values of training set size: $[100, 320, 1000...10^5]$. The training samples were generated using an exact sampler. The number of test samples were fixed to be $10^5$ across all three experiments. We find that the NeuRISE-based denoising estimator shows the sharpest decay in total variation (TV) distance with increase in sample size, and performs better than the GGM based estimator of \cite{varma2024glauber}, the SEDD approach proposed in \cite{lou2024discrete} and the discrete flow matching method of \cite{gat2024discrete}. For the model presented in \cite{lou2024discrete}, the configurations had to be one-hot coded to make the algorithm work. The model D3PM \cite{austin2021structured} performs well for low number of samples but its performance deteriorates as the size of the training set is increased. Interestingly, D3PM does not show monotonic decay of TV as the number of training samples decrease. In each case, we also compute the difference between the cross-correlation matrices of the generated samples and the test data, where the correlations are defined as $C_{ij} = \frac{1}{N}\sum_{k=1}^{N} \sigma_i^{(k)} \sigma_j^{(k)}$. The decay of cross-correlation errors show a similar trend as that for the TV. Cross-correlation metric has the advantage of tractability for larger models where TV can't be efficiently computed.

Another important study that was performed is the comparison of the NeuRISE-based diffusion for different choice of noise parameter. From our experiments, it doesn't seem like the denoising scheme with soft noise significantly outperform harsh noise setting, which corresponds to autoregressive generation. In fact, for small training sets, the harsh noise case uniformly performs better than the other schemes: see Appendix \ref{sec:SuppNum} for details.




\begin{figure*}[t!]
    \centering

    \begin{subfigure}[t]{0.2\textwidth}
        \centering
        \includegraphics[width=0.6\linewidth]{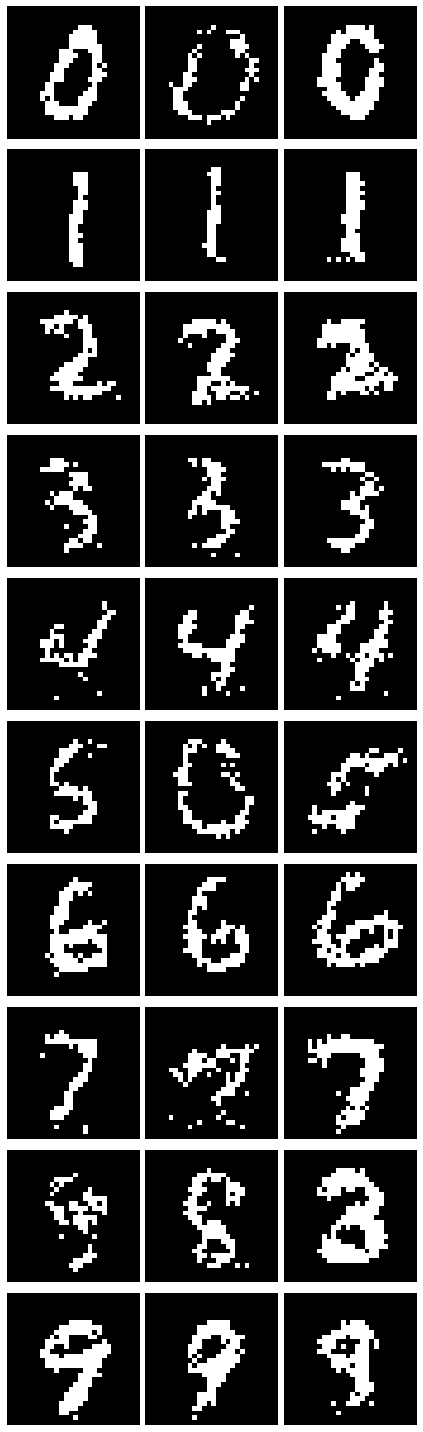}
        \caption{NeurISE Diffusion}
    \end{subfigure}
    \hfill
    \begin{subfigure}[t]{0.12\textwidth}
        \centering
        \includegraphics[width=\linewidth]{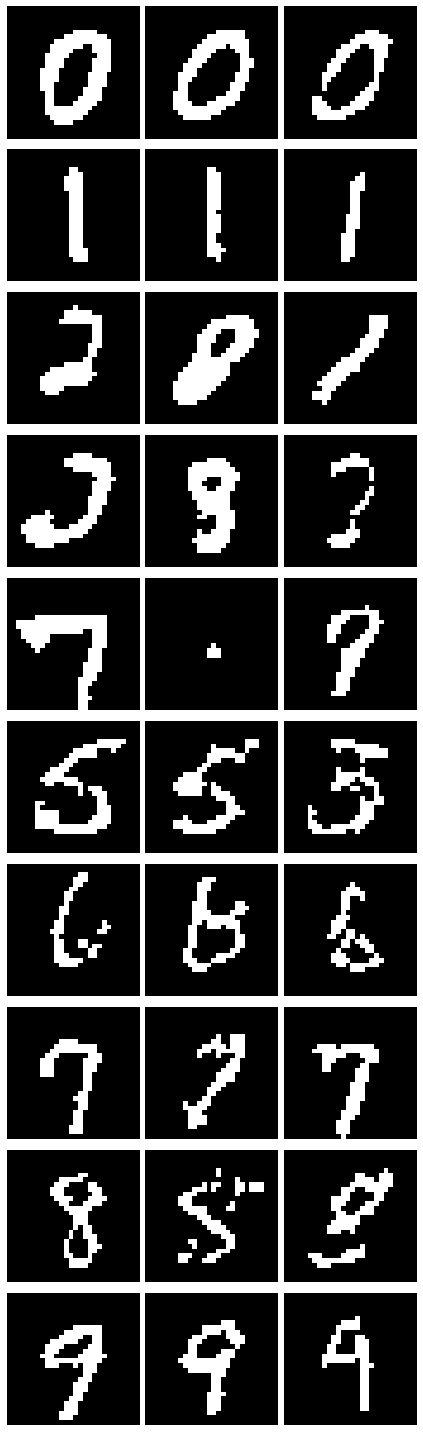}
        \caption{D3PM}
    \end{subfigure}
    \hfill
    \begin{subfigure}[t]{0.12\textwidth}
        \centering
        \includegraphics[width=\linewidth]{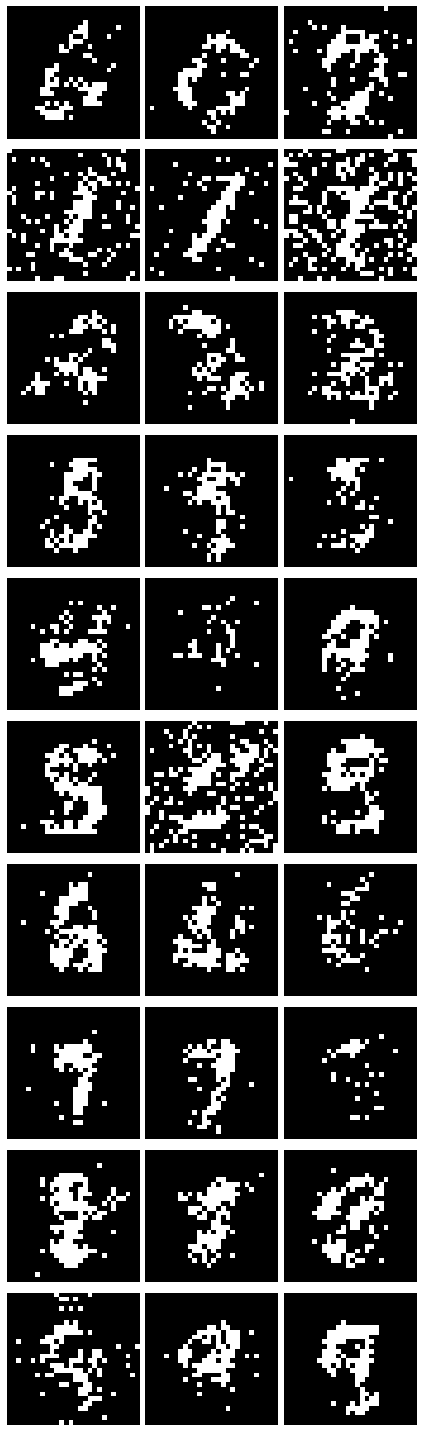}
        \caption{SEDD}
    \end{subfigure}
    \hfill
    \begin{subfigure}[t]{0.12\textwidth}
        \centering
        \includegraphics[width=\linewidth]{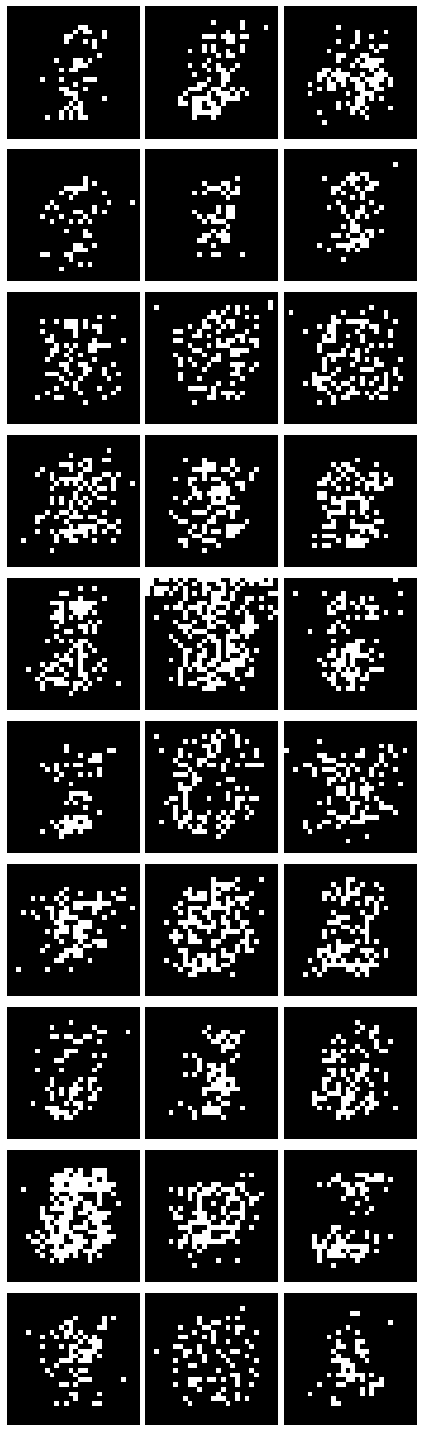}
        \caption{GGM}
    \end{subfigure}
    \hfill
    \begin{subfigure}[t]{0.12\textwidth}
        \centering
        \includegraphics[width=\linewidth]{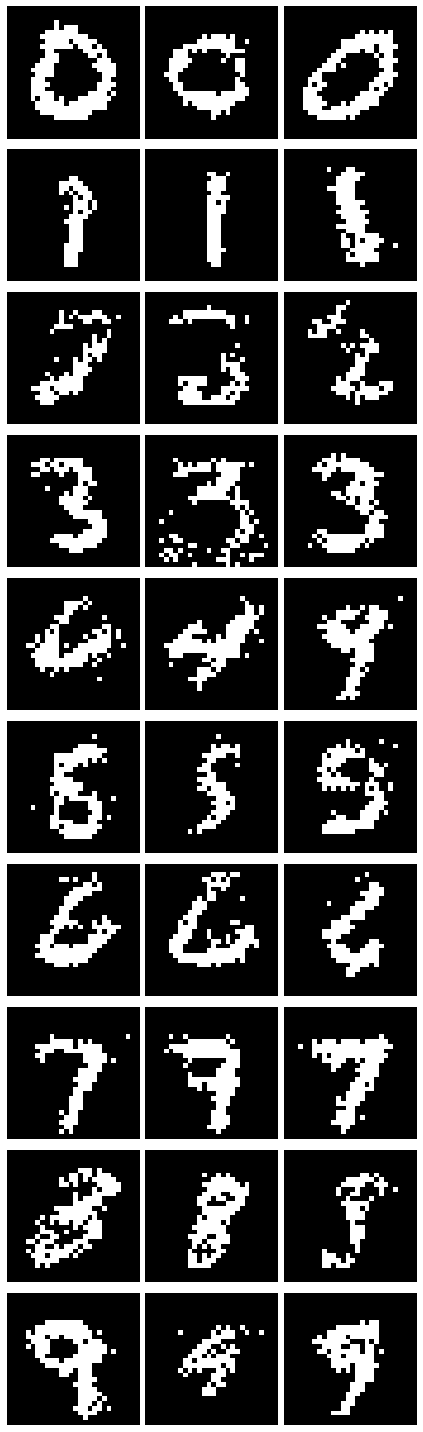}
        \caption{DFM}
    \end{subfigure}

    \caption{Class-conditional MNIST samples. Each subfigure shows generated samples arranged with one row per digit, from 0 to 9.}
    \label{fig:mnist_comparison}
\end{figure*}

\paragraph{MNIST} We evaluate the proposed discrete NeurISE diffusion model on the binarized MNIST dataset, which consists of grayscale images of handwritten digits. Images are discretized into a binarized alphabet by thresholding pixel intensities at a fixed midpoint value, assigning pixels to one of two categories depending on whether their intensity lies above or below the threshold. This results in binary-valued vectors in $\Sigma^q$, where $q = 784$ denotes the number of pixels in each image. For this benchmark, we allowed a hyperoptimization schedule to search over MLPs of upto $5$ layers. We used the MMD metric \cite{gretton2012kernel} and average cross-correlation of the samples, to compare the performance of the models for conditional sampling task. As can be seen in \tableautorefname~\ref{tab:MNIST}, the NeuRISE based learning of conditionals achieve the lowest MMD and cross-correlation error. Samples of generated images can be seen in Figure \ref{fig:mnist_comparison}. Compared to the EA model, we see that D3PM and DFM achieves a much better comparable performance in this setting. Note that the metrics used here do not compute a worst case error between distributions as TV does. This indicates that D3PM and DFM are good at reproducing a lower-order projection of the dataset that aligns with such metrics, but maybe weaker at true distribution learning. 

\paragraph{D-Wave Dataset} To demonstrate our method on a scientific application with real data, we use the diffusion model to learn a binary dataset produced by D-Wave's Advantage quantum annealer \cite{McGeochFarre2020AdvantageOverview}. This dataset is generated by performing repeated quantum annealer runs on the D-Wave machine, with a randomized set of input Hamiltonians. For our experiments we choose  $q=2000$ qubits, which forms a subsection of the annealer and train a diffusion model on the data produced by this portion of the chip. In this example, $8 \times 10^{4}$ samples were used for training, and $2 \times 10^{4}$ samples were reserved for testing. Results in \tableautorefname~\ref{tab:dwave} again show the advantage of NeurISE Diffusion in all metrics.

\paragraph{Comment on performance of various methods} We note that D3PM shows saturation in its error curves, rather than the expected decay with increasing training sample size, for the $5 \times 5$ Ising benchmark. On the other hand, it performs well in single-instance MNIST and the D-Wave data sets. This suggests less stable finite-sample scaling under these metrics, possibly because D3PM optimizes a hybrid objective consisting of a variational-bound term and an auxiliary denoising cross-entropy term. While the cross-entropy component encourages accurate denoising, the combined variational nature of the objective might lead to error saturation. Conversely, GGM shows a more regular sample-size trend on the 5×5 Ising benchmark, but its absolute performance is comparably poorer on the higher-dimensional MNIST and D-Wave datasets. We emphasize that these observed features are not a consequence of an unfortunate hyperparameter choice as hyperparameters for all methods have been optimized using \emph{hyperopt} package, see Appendix \ref{sec:implementation} for details. 

\begin{table*}[t]
    \centering
    \small

    \begin{subtable}[t]{0.48\textwidth}
        \centering
        \begin{tabular}{lcc}
            \toprule
            \textbf{Model} & \textbf{Avg. MMD} & \textbf{Avg. Correlation} \\
            \midrule
            Neurise Diff. & $1.0\mathrm{e}{-2}$ & $1.5\mathrm{e}{-6}$ \\
            D3PM & $1.5\mathrm{e}{-2}$ & $2.8\mathrm{e}{-6}$ \\
            SEDD & $4.1\mathrm{e}{-1}$ & $7.1\mathrm{e}{-6}$ \\
            GGM & $9.1\mathrm{e}{-1}$ & $8.2\mathrm{e}{-5}$ \\
            DFM & $3.0\mathrm{e}{-2}$ & $3.4\mathrm{e}{-6}$ \\
            \bottomrule
        \end{tabular}
        \caption{MNIST dataset comparison.}
        \label{tab:MNIST}
    \end{subtable}
    \hfill
    \begin{subtable}[t]{0.48\textwidth}
        \centering
        \begin{tabular}{lcc}
            \toprule
            \textbf{Model} & \textbf{MMD} & \textbf{Avg. Correlation} \\
            \midrule
            Neurise Diff. & $0.16$ & $1.18\mathrm{e}{-5}$ \\
            D3PM & $0.28$ & $1.81\mathrm{e}{-5}$ \\
            SEDD & $65.03$ & $5.81\mathrm{e}{-5}$ \\
            GGM & $96.4$ & $6.56\mathrm{e}{-5}$ \\
            DFM & $0.22$ & $2.3\mathrm{e}{-5}$ \\
            \bottomrule
        \end{tabular}
        \caption{D-Wave dataset comparison.}
        \label{tab:dwave}
    \end{subtable}

    \caption{Comparison of generative models on MNIST and D-Wave datasets.}
    \label{tab:mnist_dwave_side_by_side}
\end{table*}

\subsection{Test case 2: Multi-Alphabet Potts Models}
To demonstrate the consistency of our method in the multi-alphabet case, we consider the Potts version of the EA model Subsection \ref{ssec:EAm}. Let $\Sigma = \{0,1,\ldots,p-1\}$ and let
$\sigma = (\sigma_1,\ldots,\sigma_q) \in \Sigma^q$ with $p=L^2$.
We consider a $q$-state Potts model on an $L\times L$ periodic lattice with
Hamiltonian, $
H(\sigma)
=
- \sum_{(i,j)\in E} J_{ij}\,\mathbf{1}\{\sigma_i = \sigma_j\}
- \sum_{i=1}^{p} h_{i,\sigma_i},
$,
where $E$ denotes the set of nearest-neighbor pairs on the lattice. Here 
$J_{ij}=J_{ji}\in\{-J,+J\}$ are random couplings,
$h_{i,s}\in\{-h,+h\}$ are state-dependent local fields,
and $\mathbf{1}\{\cdot\}$ is the indicator function. 
We test the model for two lattices, $L = 2$ and $L = 3$, which corresponds to $q = 4$ and $q =9$ states, respectively. As can be seen in Figure \ref{fig:Multiafig}, the TV error decreases as the number of training samples are increased.


    
  



\begin{figure}[t]
    \centering
    \vspace{-0.3cm}

    \begin{subfigure}[t]{0.47\columnwidth}
        \centering
        \includegraphics[width=\linewidth]{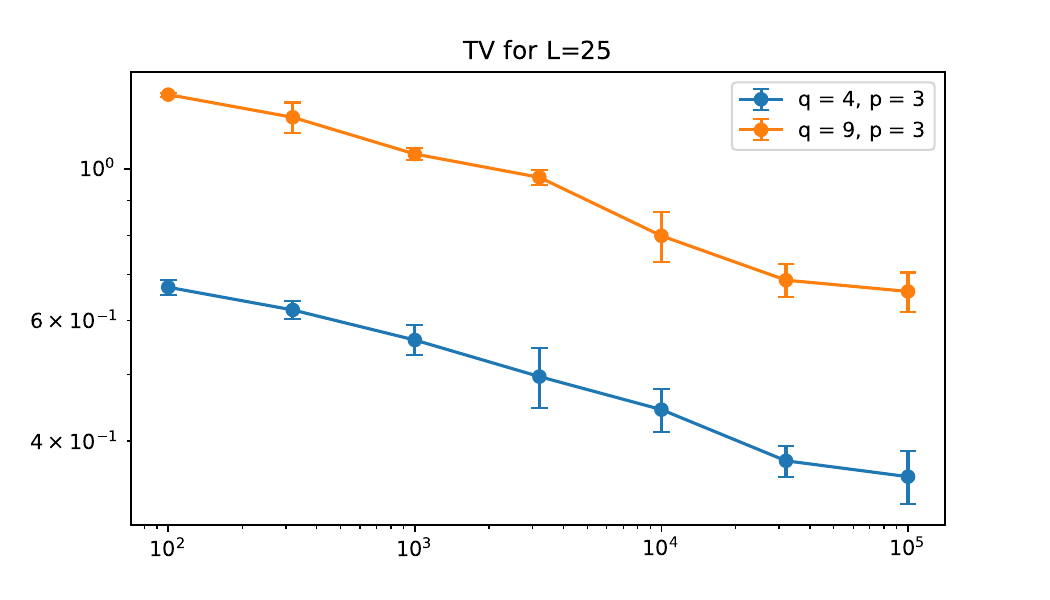}
        \caption{Non-binary Potts model.}
        \label{fig:Multiafig}
    \end{subfigure}
    \hfill
    \begin{subfigure}[t]{0.43\columnwidth}
        \centering
        \includegraphics[width=\linewidth]{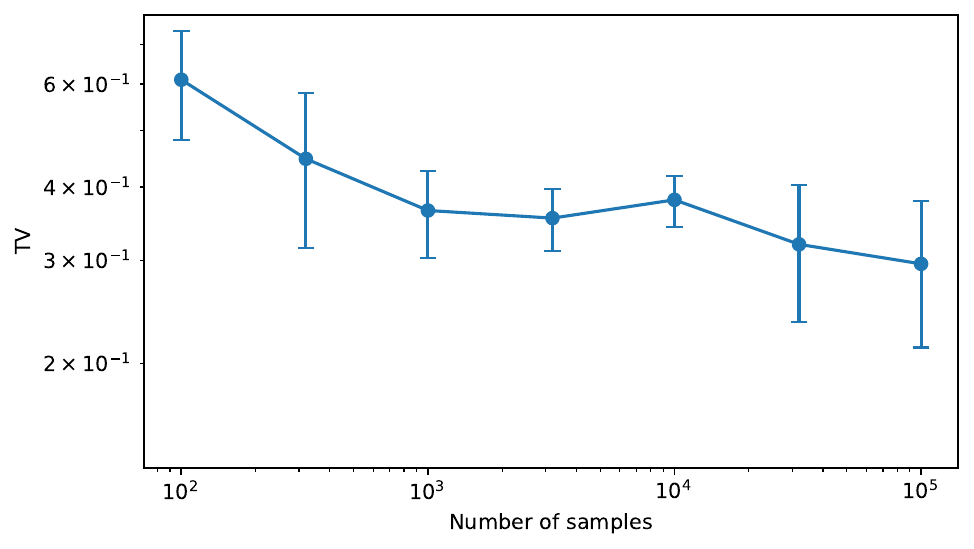}
        \caption{GHZ state.}
        \label{fig:MultiaQuantum}
    \end{subfigure}
    \vspace{-0.1cm}
    \caption{
    Trends for NeuRISE Diffusion as a function of training set size, averaged over $10$ trials for each size.
    The test sample size was $10^5$ for each experiment.
    Error bars represent one standard deviation over trials.
    Left: TV for a non-binary Potts model.
    Right: cross-correlation error of NeuRISE Diffusion trained to learn the GHZ state. 
    }
    \vspace{-0.45cm}
    \label{fig:multi_potts_quantum}
\end{figure}

\subsection{Test case 3: Quantum Tomography of GHZ state}
To test our methods for the multi-alphabet case for a scientifically relevant applications, we use quantum tomography data obtained from the simulation of a four-outcome measurement $(p=4)$ on the Greenberger–Horne–Zeilinger (GHZ) state. This dataset is commonly used in the study of neural net based approaches to the representation of quantum states \cite{torlai2018neural, jayakumar2024learning}. We study the efficacy of NeuRISE Diffusion on this model with $20$ qubits ($q=20$) in \figureautorefname~\ref{fig:MultiaQuantum}. The cross-correlation is generalized to the multi-alphabet case as $C_{ij} = \frac{1}{N}\sum_{k=1}^{N} \sum_{a \in \Sigma} \delta_{\sigma^{(k)}_i = a}\delta_{\sigma^{(k)}_j = a} $. We see that the cross-correlation error goes down significantly after $10^4$ samples, indicating that the model is able to learn a faithful generative model for this quantum state.

\section*{Conclusion}
We introduced a discrete diffusion framework that combines round-robin single-site noising with Neural Interaction Screening (NeurISE) to model high-dimensional categorical data. By learning single-site conditional distributions at intermediate diffusion steps, the proposed approach enables an sample efficient reverse denoising process without requiring full joint likelihood estimation. Empirical results on a variety of synthetic and scientific datasets demonstrate that the method effectively captures complex dependency structures in both image-based and physically motivated discrete systems. Our code is provided in the \href{https://github.com/lanl-ansi/NeurISEdiffusion}{Github repository}.

\subsection*{Acknowledgment}
This work has been supported by the U.S. Department of Energy/Office of Science Advanced Scientific Computing Research Program.



\bibliography{ddb}
\bibliographystyle{plain}

\newpage 
\appendix
\onecolumn

\section{Pseudcode for NeurISE Diffusion}
In this section, we present the Neurise based denoising diffusion algorithm introduced in the paper.
\begin{algorithm}
\caption{Discrete Diffusion with NeurISE} 
\label{alg:neurise_full_random_t}
\begin{algorithmic}[1]
\STATE \textbf{Input:} alphabet $\Sigma$ with $|\Sigma|=p$, dimension $q$, steps $T$ , noise $\varepsilon\in[0,1]$, data distribution $\mu_0$ (samples $, \sigma_0\sim\mu_0$), Number of samples $N$.
\vspace{0.4em}

\textbf{Forward diffusion }
\STATE Sample time index $t \sim \mathrm{Unif}(\{1,2,\ldots,T\})$
\STATE Initialize $\sigma \gets \sigma_0$
\FOR{$n = 1, 2, \ldots, N$}
    \STATE $u \gets ((t-1) \bmod q) + 1$
    \STATE Set $\sigma_{-u} \gets \sigma_{-u}$ 
    \STATE With probability $1-\varepsilon$ change coordinate $u$ according  $\sigma_u\sim \mathrm{Unif}(\Sigma)$
\ENDFOR
\STATE Output forward tuple $(t, \sigma_0, \sigma_t)$ where $\sigma_t \gets \sigma$

\STATE \textbf{Learn conditionals with NeurISE.}
\STATE \textbf{Goal:} estimate single-site conditionals $\hat{\mu}_s(\cdot \mid \sigma_{s,-u})$ for $s=0,\ldots,T-1$
\FOR{$s = 0, 1, \ldots, T-1$}
    \STATE Generate noised samples at time $s$ by running the forward kernel on data to obtain a batch $\{\sigma_s^{(n)}\}_{n=1}^N$
    \FOR{$u = 1, 2, \ldots, q$}
        \STATE Train NeurISE network  $\mathrm{NN}_{\theta}$ by minimizing the NeurISE objective \eqref{eq:NISE}
        \STATE Obtain conditional estimator $\hat{\mu}_s(\cdot\mid\sigma_{s,-u})$  via  \eqref{eq:condexp}
    \ENDFOR
\ENDFOR

\vspace{0.7em}
\STATE \textbf{Reverse sampling (denoising).}
\STATE Initialize $\tilde{\sigma}_T \sim \mathrm{Unif}(\Sigma^q)$
\FOR{$r = T-1, T-2, \ldots, 0$}
    \STATE $u \gets (r \bmod q) + 1$
    \STATE Set $\hat{k}^{rev}_n$ according to \eqref{eq:k_rev1}-\eqref{eq:k_rev2}
    \STATE Sample $\tilde{\sigma}_{T-1} \sim \hat{k}^{rev}_n(\sigma_{T}, \cdot)$
\ENDFOR
\STATE \textbf{return} $\tilde{\sigma}_0$
\end{algorithmic}
\end{algorithm}

\section{Robustness Study for NeurISE Diffusion}
\label{sec:SuppNum}

\subsection{Soft Noise vs Harsh Noise}

In this section, we show our comparison of the harsh noise setting with soft noise ones, to facilitate comparison of autoregressive vs diffusion model. From our observation, the harsh noise setting performs better at low number of training samples and for large number of training samples, the models show very similar performance. See Figure \ref{fig:harsh_vs_soft}. Here, $N$ denotes the total number of time-steps used in the noising and denosing phase, and $e$ denotes the noise parameter $\varepsilon$.

\begin{figure}[!htb]
    \centering
    \begin{subfigure}[b]{0.45\textwidth}
        \centering
        \includegraphics[width=\textwidth]{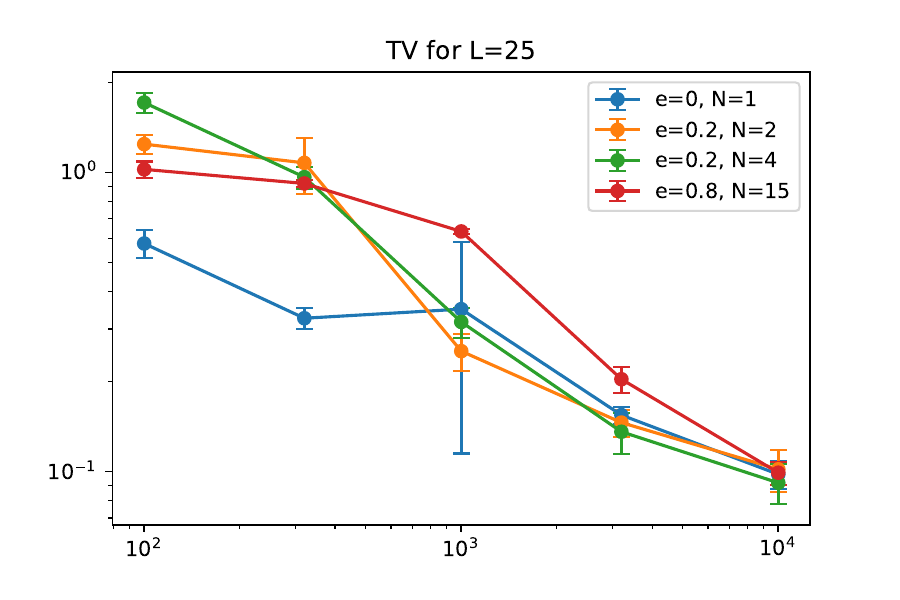}
        \caption{TV}
        \label{fig:subfig22a}
    \end{subfigure}
    \hfill
    \begin{subfigure}[b]{0.45\textwidth}
        \centering
        \includegraphics[width=\textwidth]{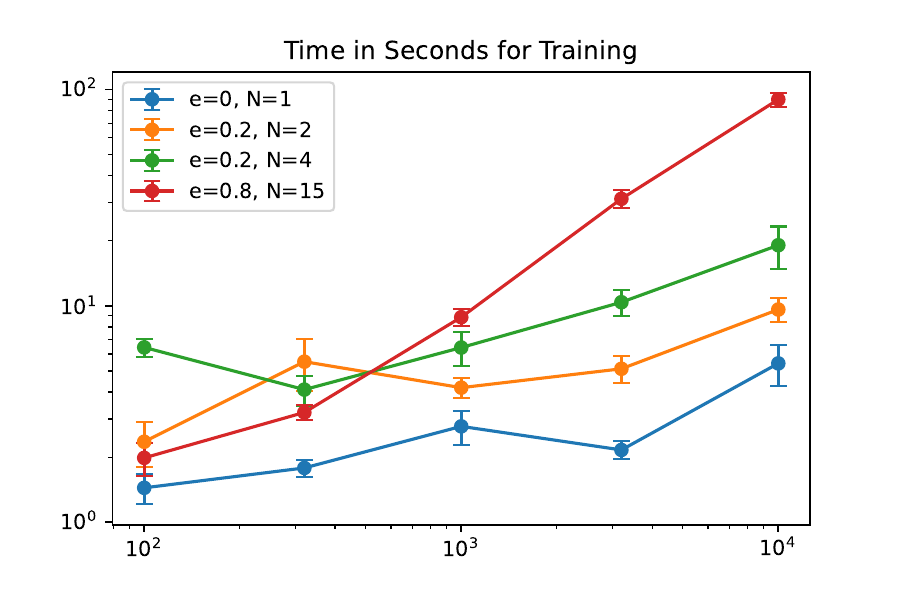}
        \caption{Training Time}
        \label{fig:subfig22b}
    \end{subfigure}
    \caption{Trend of TV and Training time for Ising models, averaged $10$ trials per data set. The test sample size was taken to be $10^4$ for each experiment. The harsh noise version of the problem performs competitively with situations where noise is soft.}
    \label{fig:harsh_vs_soft}
\end{figure}

\subsection{Local vs Global Neural Network}
In this section, we compare the performance of diffusion models trained using two architectures: (i) a collection of local neural networks, with one network per time step, and (ii) a single global neural network shared across the entire time horizon. As shown in Fig.~\ref{fig:global_vs_local}, the local architecture consistently outperforms the global one, despite both approaches having approximately the same total number of trainable parameters. Specifically, in the local setting, each time step is modeled by a single-hidden-layer MLP with 5 hidden units, whereas the global model uses a fixed network of width 125.

We tested NeuRISE Diffusion for two different lattices: $2 \times 2$  and $3 \times 3$, each with alphabet size $3$. As can be seen Figure \ref{fig:Multiafig}, the TV decreases in a statistically expected way as the number of training samples is increased from $10^2$ to $10^5$.

\begin{figure}[!htb]
    \centering
    \begin{subfigure}[b]{0.45\textwidth}
        \centering
        \includegraphics[width=\textwidth]{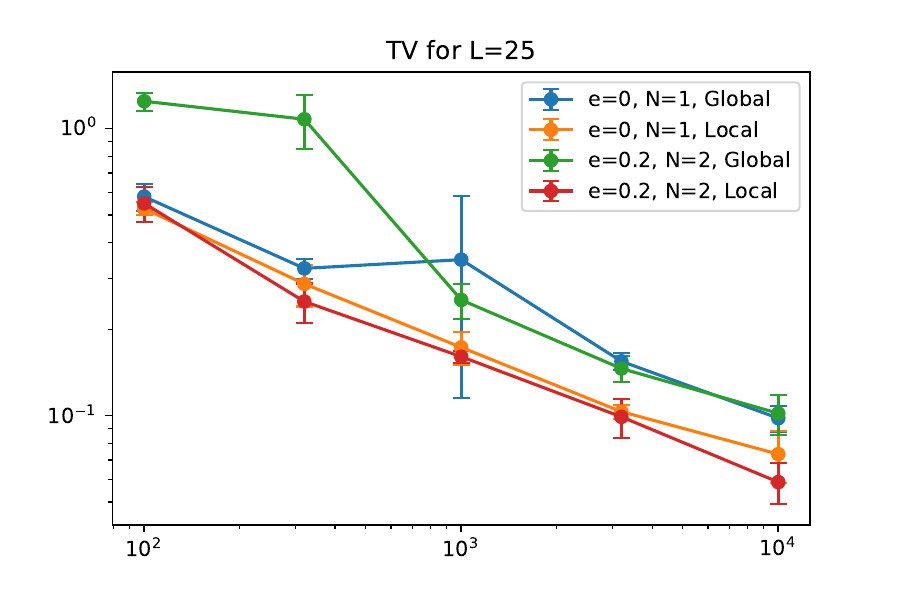}
        \caption{TV}
        \label{fig:subfig32a}
    \end{subfigure}
    \hfill
    \begin{subfigure}[b]{0.45\textwidth}
        \centering
        \includegraphics[width=\textwidth]{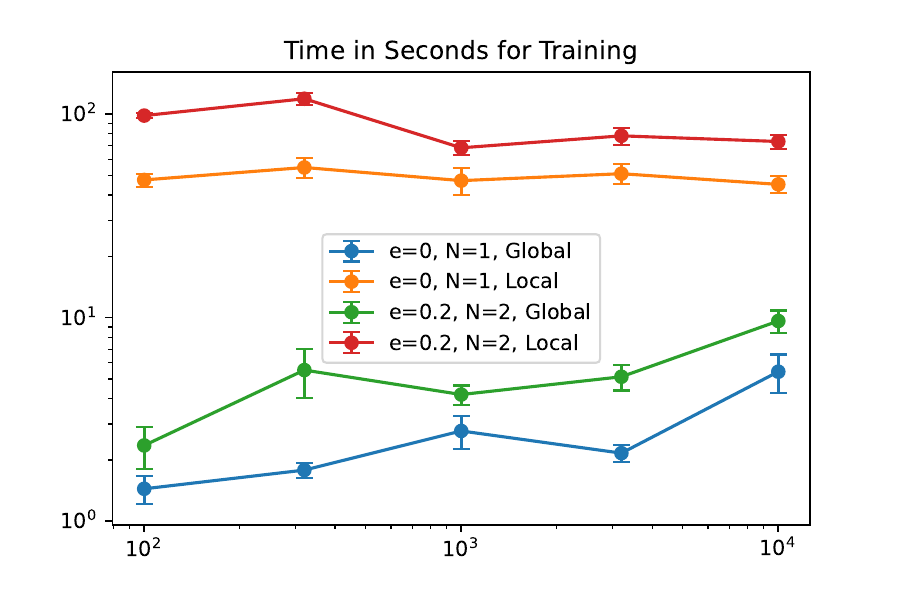}
        \caption{Training Time}
        \label{fig:subfig32b}
    \end{subfigure}
    \caption{Trend of TV and Training time for Ising models, averaged $10$ trials per data set. The test sample size was taken to be $10^4$ for each experiment. Learning individual conditionals using separate neural networks achieves lower TV, than using one global neural network across all conditionals.}
    \label{fig:global_vs_local}
\end{figure}

\section{Theory}
\label{sec:theory}
For the presentation in this section, we introduce some additional notation. 
Given a probability distribution $\mu$ on $\Sigma^q$, the action of the reverse kernel
$k_n^{\mathrm{rev}}$ on $\mu$ is defined by
\[
(\mu k_n^{\mathrm{rev}})(\tilde{\sigma})
=
\sum_{\sigma\in\Sigma^q}
\mu(\sigma)\, k_n^{\mathrm{rev}}(\sigma,\tilde{\sigma}),
\qquad \tilde{\sigma}\in\Sigma^q.
\]

For two kernels $k_n^{\mathrm{rev}}$ and $k_{n-1}^{\mathrm{rev}}$, their
composition is the kernel
\[
(k_n^{\mathrm{rev}} k_{n-1}^{\mathrm{rev}})(\sigma,\sigma')
=
\sum_{\tilde{\sigma}\in\Sigma^q}
k_n^{\mathrm{rev}}(\sigma,\tilde{\sigma})\,
k_{n-1}^{\mathrm{rev}}(\tilde{\sigma},\sigma'),
\]
which corresponds to applying $k_n^{\mathrm{rev}}$ first and
$k_{n-1}^{\mathrm{rev}}$ second.

More generally, the composed reverse kernel
\[
\mathcal{K}_{0:T-1}^{\mathrm{rev}}
=
k_{T-1}^{\mathrm{rev}}\cdots k_0^{\mathrm{rev}}
\]
satisfies
\[
(\mathcal{K}_{0:T-1}^{\mathrm{rev}})(\sigma_T,\sigma_0)
=
\sum_{\sigma_1,\dots,\sigma_{T-1}\in\Sigma^q}
\prod_{t=0}^{T-1}
k_t^{\mathrm{rev}}(\sigma_{t+1},\sigma_t),
\]
where $\sigma_T$ denotes the state at time $T$ and $\sigma_0$ the state at time
$0$.

Finally, the distribution obtained by initializing the reverse chain from
$\mu_{\mathrm{noise}}$ is
\begin{align*}
&\nu_0(\sigma_0)
=
(\mu_{\mathrm{noise}} \mathcal{K}_{0:T-1}^{\mathrm{rev}})(\sigma_0) \\
& =
\sum_{\sigma_T\in\Sigma^q}
\mu_{\mathrm{noise}}(\sigma_T)
(\mathcal{K}_{0:T-1}^{\mathrm{rev}})(\sigma_T,\sigma_0).
\end{align*}





The following result provides analogue of the convergence guarantees established for
score-based diffusion models in the continuous time 
\cite{de2021diffusionSB,chen2022sampling,chen2024convergence,zhang2024convergence,ren2024discrete}. Notably, the analysis is much simpler than in \cite{de2021diffusionSB,chen2022sampling,chen2024convergence,zhang2024convergence,ren2024discrete} due to the discrete time setting. The key insight is that the
discrepancy between the output distribution of the approximate reverse chain and the true
data distribution decomposes cleanly into two contributions: (i) the extent to which the
forward process has mixed toward the noise distribution, and (ii) the cumulative error
incurred when approximating the reverse kernels. This decomposition makes precise the
tradeoff underlying DDPM-style generative modeling. Accurate sampling requires both
sufficiently fast diffusion of the forward process to the noise distribution and sufficiently accurate estimation of the reverse-time
dynamics. Here, the total variation (TV) distance $\|\cdot\|_{\rm TV}$ between two distributions is defined by $
\|\hat{\mu}-\mu\|_{\mathrm{TV}}
=\frac12\sum_{\sigma \in \Sigma^q}
\bigl|\hat{\mu}(\sigma)-\mu(\sigma)\bigr|.$

\begin{theorem}
\label{thm:errorbndmain}
Let $\{X_n\}_{n=0}^T$ be the 
Markov chain on $\Sigma^q$ with forward transition kernels
$k_n:\Sigma^q\times\Sigma^q\to\mathbb{R}_{\ge 0}$. 
Fix a noise reference distribution $\mu_{\mathrm{noise}}$ on $\Sigma^q$ and assume that for some
$\delta_T\in[0,1]$,
\begin{equation}
\lVert \mu_T - \mu_{\mathrm{noise}} \rVert_{\mathrm{TV}} \le \delta_T.
\end{equation}

Let $\{k_n^{\mathrm{rev}}\}_{n=0}^{T-1}$ be a well-defined family of reverse kernels that satisfy \eqref{eq:revp}. Consider approximate reverse kernels
$\{\widehat{k}_n^{\mathrm{rev}}\}_{n=0}^{T-1}$ such that for all $n=0,\dots,T-1$,
\begin{equation}
\sup_{\sigma \in\Sigma^q}
\left\lVert
\widehat{k}_n^{\mathrm{rev}}(\cdot, \sigma) - k_n^{\mathrm{rev}}(\cdot,  \sigma)
\right\rVert_{\mathrm{TV}}
\le \eta.
\end{equation}
Initialize the approximate reverse chain with the noise reference, i.e.\ $Y_T\sim \mu_{\mathrm{noise}}$,
and let $\widehat{\mu}_0$ denote the law of the output $Y_0$ obtained by applying
$\widehat{k}_{T-1}^{\mathrm{rev}},\dots,\widehat{k}_0^{\mathrm{rev}}$.

Then the output distribution satisfies 
\begin{equation}
\|\hat{\mu}_0 - \mu_0\|_{\mathrm{TV}}
\;\le\;
\underbrace{\delta_T}_{\text{Mixing error}}
\;+\;
\underbrace{T\,\eta.}_{\text{Reverse kernel estimation error}}
\end{equation}
\end{theorem}
\begin{proof}
Let $\mathcal{K}_{0:T-1}^{\mathrm{rev}}$ denote the composition of the exact reverse kernels
$k_{T-1}^{\mathrm{rev}},\dots,k_0^{\mathrm{rev}}$ (applied in this order), and let
$\widehat{\mathcal{K}}_{0:T-1}^{\mathrm{rev}}$ denote the composition of the approximate reverse kernels
$\widehat{k}_{T-1}^{\mathrm{rev}},\dots,\widehat{k}_0^{\mathrm{rev}}$.

Let $\nu_0$ be the law of the output obtained by running the \emph{exact} reverse chain initialized at
time $T$ from $\mu_{\mathrm{noise}}$, i.e.
\[
\nu_0 := \mu_{\mathrm{noise}} \, \mathcal{K}_{0:T-1}^{\mathrm{rev}}.
\]
Since the kernels $\{k_n^{\mathrm{rev}}\}$ satisfy \eqref{eq:revp}, initializing the exact reverse chain
from $\mu_T$ yields $\mu_0$. Hence,
\[
\mu_0 = \mu_T \, \mathcal{K}_{0:T-1}^{\mathrm{rev}}.
\]
By the data processing inequality 
\[
\lVert \nu_0 - \mu_0 \rVert_{\mathrm{TV}}
\le
\lVert \mu_{\mathrm{noise}} - \mu_T \rVert_{\mathrm{TV}}
\le \delta_T.
\tag{IE}
\]

Let $\widehat{\mu}_0$ be the law of the output of the \emph{approximate} reverse chain initialized from
$\mu_{\mathrm{noise}}$, i.e.
\[
\widehat{\mu}_0 := \mu_{\mathrm{noise}} \, \widehat{\mathcal{K}}_{0:T-1}^{\mathrm{rev}}.
\]
We bound $\lVert \widehat{\mu}_0 - \nu_0 \rVert_{\mathrm{TV}}$ by a telescoping argument as is used in pertubation theory of Markov chains \cite{rudolf2024perturbations}.
Define intermediate distributions for $m=0,1,\dots,T$:
\[
\rho^{(m)} :=
\mu_{\mathrm{noise}}\,
\widehat{k}_{T-1}^{\mathrm{rev}}\cdots \widehat{k}_{T-m}^{\mathrm{rev}}\,
k_{T-m-1}^{\mathrm{rev}}\cdots k_0^{\mathrm{rev}},
\]
with the convention that $\rho^{(0)}=\nu_0$ and $\rho^{(T)}=\widehat{\mu}_0$. Then by the triangle inequality,
\[
\lVert \widehat{\mu}_0 - \nu_0 \rVert_{\mathrm{TV}}
=
\lVert \rho^{(T)} - \rho^{(0)} \rVert_{\mathrm{TV}}
\le
\sum_{m=1}^{T}
\lVert \rho^{(m)} - \rho^{(m-1)} \rVert_{\mathrm{TV}}.
\]
Fix $m\in\{1,\dots,T\}$ and set
\[
\alpha^{(m)} :=
\mu_{\mathrm{noise}}\,
\widehat{k}_{T-1}^{\mathrm{rev}}\cdots \widehat{k}_{T-m+1}^{\mathrm{rev}},
\]
so that $\rho^{(m)} = \alpha^{(m)} \widehat{k}_{T-m}^{\mathrm{rev}} k_{T-m-1}^{\mathrm{rev}}\cdots k_0^{\mathrm{rev}}$
and $\rho^{(m-1)} = \alpha^{(m)} k_{T-m}^{\mathrm{rev}} k_{T-m-1}^{\mathrm{rev}}\cdots k_0^{\mathrm{rev}}$.
Using contraction of TV under a common kernel,
\[
\lVert \rho^{(m)} - \rho^{(m-1)} \rVert_{\mathrm{TV}}
\le
\left\lVert
\alpha^{(m)} \widehat{k}_{T-m}^{\mathrm{rev}} - \alpha^{(m)} k_{T-m}^{\mathrm{rev}}
\right\rVert_{\mathrm{TV}}.
\]
For any distribution $\alpha$ and kernels $P,Q$ on $\Sigma^q$,
\[
\lVert \alpha P - \alpha Q \rVert_{\mathrm{TV}}
\le
\sup_{\sigma \in\Sigma^q} \lVert P(\cdot, \sigma) - Q(\cdot, \sigma) \rVert_{\mathrm{TV}}.
\]
Applying this with $\alpha=\alpha^{(m)}$, $P=\widehat{k}_{T-m}^{\mathrm{rev}}$, $Q=k_{T-m}^{\mathrm{rev}}$ and the estimation error, we get,
\[
\lVert \rho^{(m)} - \rho^{(m-1)} \rVert_{\mathrm{TV}} \le \eta.
\tag{KE}
\]
Combining gives
\[
\lVert \widehat{\mu}_0 - \nu_0 \rVert_{\mathrm{TV}} \le T\eta.
\]

Finally, by the triangle inequality,
\begin{align*}
\lVert \widehat{\mu}_0 - \mu_0  \rVert_{\mathrm{TV}}
& \le
\lVert \widehat{\mu}_0 - \nu_0 \rVert_{\mathrm{TV}}
+
\lVert \nu_0 - \mu_0 \rVert_{\mathrm{TV}} \\
& \le
T\eta + \delta_T,
\end{align*}
which concludes the proof.
\end{proof}

In the following Theorem we bound the error of the approximate reverse chain when
initialized from the true noise distribution \(\mu_{\mathrm{noise}}\). In practice,
however, the reverse process is initialized from an empirical approximation $\frac{1}{N}  \sum_{i=1}^N \delta_{X^{\rm data}_{i}}$
of \(\mu_{\mathrm{noise}}\). The following corollary shows that this additional source of
error contributes additively to the final bound and captures the effect of sampling error from the noise distribution. It partially explains why masked diffusion models have been observed to perform better \cite{austin2021structured,santos2023blackout,lou2024discrete} in practice, than when the noise distribution is uniform.

\begin{corollary}(Initialization error)
\label{cor:empinierro}
In the setting of Theorem~\ref{thm:errorbndmain}, let
$\widehat{\mu}_{\mathrm{noise}}$ be any distribution on $\Sigma^q$ such that for
some $\gamma\in[0,1]$,
\begin{equation}
\label{eq:noise_approx}
\bigl\|
\widehat{\mu}_{\mathrm{noise}}-\mu_{\mathrm{noise}}
\bigr\|_{\mathrm{TV}}
\le \gamma .
\end{equation}
Initialize the approximate reverse chain with
$Y_T \sim \widehat{\mu}_{\mathrm{noise}}$, and let $\widetilde{\mu}_0$ denote the
law of the output $Y_0$ obtained by applying
$\widehat{k}_{T-1}^{\mathrm{rev}},\dots,\widehat{k}_0^{\mathrm{rev}}$.

Then the output distribution satisfies
\begin{align}
\bigl\|
\widetilde{\mu}_0 - \mu_0
\bigr\|_{\mathrm{TV}}
\;\le\;
&\underbrace{\delta_T}_{\text{Mixing Error}} + \underbrace{T\eta}_{\text{Reverse kernel estimation error}} \nonumber \\ 
& + \underbrace{\gamma.}_{\text{Noise sampling error}}
\label{eq:empirical_init_bound}
\end{align}
\end{corollary}
\begin{proof}
Let $\widehat{\mathcal K}^{\mathrm{rev}}_{0:T-1}$ denote the composition of the
approximate reverse kernels. Define
\[
\widehat{\mu}_0
:=
\mu_{\mathrm{noise}}\,
\widehat{\mathcal K}^{\mathrm{rev}}_{0:T-1},
\qquad
\widetilde{\mu}_0
:=
\widehat{\mu}_{\mathrm{noise}}\,
\widehat{\mathcal K}^{\mathrm{rev}}_{0:T-1}.
\]
The TV norm under the action of a Markov kernel $Q$ remains preserved (this follows trivially from $\sum_{\sigma \in \Sigma^q}Q(\sigma,\tilde{\sigma})$ = 1) and hence,
\begin{align*}
\bigl\|
\widetilde{\mu}_0-\widehat{\mu}_0
\bigr\|_{\mathrm{TV}}
& =
\bigl\|
(\widehat{\mu}_{\mathrm{noise}}-\mu_{\mathrm{noise}})
\widehat{\mathcal K}^{\mathrm{rev}}_{0:T-1}
\bigr\|_{\mathrm{TV}} \\
& \le
\bigl\|
\widehat{\mu}_{\mathrm{noise}}-\mu_{\mathrm{noise}}
\bigr\|_{\mathrm{TV}}
\le \gamma .
\end{align*}
The claim follows by the triangle inequality together with
Theorem~\ref{thm:errorbndmain}, which gives
$\|\widehat{\mu}_0-\mu_0\|_{\mathrm{TV}}\le \delta_T + T\eta$.
\end{proof}

One can use this corollary to see the effect of error due to sampling from the noise distribution. For instance, let $\widehat{\mu}_{\mathrm{noise}} = \frac{1}{N}\sum_{i=1}^N\delta_{X^{\rm noise}_i}$ be the approximating empirical distribution based on
$N$ i.i.d.\ samples $\{ X^{\rm noise}_1,...X^{\rm noise}_N\}$ from $\mu_{\mathrm{noise}}$.
Then from results of \citep{berend2012convergence} one can quantify the effect of sampling error from the noise distribution, on the distance of the sampled distribution from the data distribution. In the special case, when $\mu^{\rm noise} = \delta_{\sigma_{\rm mask}}$ for some $\sigma_{\rm mask} \in \Sigma^q$, then it is to see that one can in fact, take a stronger bound by setting $\gamma = 0$. While this might explain partially why diffusion models with absorbing states perform better as observed in literature \cite{austin2021structured,santos2023blackout,lou2024discrete}, it can be that the estimation error of the reversal kernel, as captured by $\eta$ is high in such situations, as the distribution becomes much more concentrated. On the other hand, in experiments, we observed uniform distribution performed better. We conjecture this is due to this noising process increasing the temperature of the distribution and due to fact that higher-temperature distributions being easier to learn via NeurISE \cite{jayakumar2020learning}.

\subsection{Non-uniqueness of Reverse Processes}
\label{ssec:nonuni}
In this section, we highlight that, in general, the reverse process associated with a forward Markov chain is not unique. Even when the marginal distributions \( \mu_t \) at each time \( t \in \{0, \dots, T\} \) are fixed, there may exist multiple valid reverse dynamics that recover the same marginals.

Let \( \{X_t\}_{t=0}^T \) be a forward Markov process over \( \Sigma^q \), with \( X_t \sim \mu_t \) for each \( t \). The canonical construction of the reverse process introduced in the main text uses Bayes' rule:
\begin{align*}
\mathbb{P}(X_t = \sigma \mid X_{t+1} = \tilde{\sigma}) 
&= \frac{\mathbb{P}(X_t = \sigma, X_{t+1} = \tilde{\sigma})}{\mathbb{P}(X_{t+1} = \tilde{\sigma})} \\
&= \frac{\mathbb{P}(X_{t+1} = \tilde{\sigma} \mid X_t = \sigma) \cdot \mathbb{P}(X_t = \sigma)}{\mathbb{P}(X_{t+1} = \tilde{\sigma})}.
\end{align*}

This defines a valid reverse kernel based on the forward transition probabilities and the marginal distributions \( \mu_t \). However, other reverse processes may exist that yield the same marginals.

Let \( \{Y_t\}_{t=0}^T \) be another sequence of random variables over \( \Sigma^q \) such that:
\begin{align}
\mathbb{P}(Y_t = \sigma) = \mathbb{P}(X_t = \sigma) = \mu_t(\sigma) \quad \text{for all } t \in \{0, \dots, T\}. 
\end{align}
Then \( Y_t \) is a valid alternative reverse process if it satisfies the marginal constraints above.

\medskip

\noindent
As an illustrative example, consider the extreme case where the reverse kernel is marginally independent of the conditioning variable:
\begin{align}
\mathbb{P}(Y_t = \sigma \mid Y_{t+1} = \tilde{\sigma}) = \mu_t(\sigma). 
\end{align}

In other words, the reverse step simply resamples from the marginal \( \mu_t \), ignoring the previous state \( Y_{t+1} \).

Now fix \( Y_T = X_T \sim \mu_T \), and define \( Y_t \) recursively using (2). Then we show by induction that:
\begin{align*}
\mathbb{P}(Y_t = \sigma) &= \sum_{\hat{\sigma} \in \Sigma^q} \mathbb{P}(Y_t = \sigma \mid Y_{t+1} = \hat{\sigma}) \cdot \mathbb{P}(Y_{t+1} = \hat{\sigma}) \\
&= \sum_{\hat{\sigma} \in \Sigma^q} \mu_t(\sigma) \cdot \mu_{t+1}(\hat{\sigma}) \\
&= \mu_t(\sigma) \cdot \sum_{\hat{\sigma} \in \Sigma^q} \mu_{t+1}(\hat{\sigma}) \\
&= \mu_t(\sigma),
\end{align*}
since \( \mu_{t+1} \) is a probability distribution and thus sums to 1.

In this degenerate case, the reverse kernel \( k^{\text{rev}}_t : \Sigma^q \times \Sigma^q \to \mathbb{R}_{\geq 0} \) is defined by:
\begin{equation}
k^{\text{rev}}_t(\sigma, \tilde{\sigma}) := \mathbb{P}(Y_t = \sigma \mid Y_{t+1} = \tilde{\sigma}) = \mu_t(\sigma) \quad \forall \tilde{\sigma} \in \Sigma^q.
\end{equation}

That is, \( k^{\text{rev}}_t(\cdot, \tilde{\sigma}) \) is simply the marginal distribution \( \mu_t \), regardless of the value of \( \tilde{\sigma} \). This reverse kernel completely ignores the conditioning state and independently resamples \( \sigma \sim \mu_t \) at each step.

While this kernel does not capture the time-reversal of the actual forward dynamics, it still guarantees the correct marginal distributions at every time step:
\begin{align*}
\mu_t(\sigma) & = \sum_{\tilde{\sigma} \in \Sigma^q} k^{\text{rev}}_t(\sigma, \tilde{\sigma}) \cdot \mu_{t+1}(\tilde{\sigma})  \\
& = \mu_t(\sigma) \cdot \sum_{\tilde{\sigma} \in \Sigma^q} \mu_{t+1}(\tilde{\sigma}) = \mu_t(\sigma).
\end{align*}

This construction shows that the reverse process is not uniquely determined by the marginal sequence \( \{\mu_t\}_{t=0}^T \), and highlights a family of reverse dynamics that can be arbitrarily different from the canonical reverse Markov process.

In fact, the set of all admissible reverse kernels that satisfy the marginal condition is convex; any convex combination of two valid reverse kernels \( k^{\text{rev}}_{t,1} \) and \( k^{\text{rev}}_{t,2} \) also yields a valid reverse kernel,
\[
k^{\text{rev}}_t = \lambda k^{\text{rev}}_{t,1} + (1 - \lambda) k^{\text{rev}}_{t,2}, \quad \text{for any } \lambda \in [0, 1].
\]
This further underscores the flexibility  and ambiguity inherent in defining reverse-time dynamics.

It is important to emphasize that the transitions defined by general reverse kernels are \emph{not local}. Unlike the canonical reverse process where transitions are typically constrained to move between configurations that differ by a single spin (i.e., Hamming distance one) this degenerate reverse kernel allows transitions between any two configurations in \( \Sigma^q \), regardless of their Hamming distance:
\[
k^{\text{rev}}_t(\sigma, \tilde{\sigma}) = \mu_t(\sigma) \quad \text{for all } \sigma, \tilde{\sigma} \in \Sigma^q.
\]

In other words, starting from any configuration \( \tilde{\sigma} \), the reverse process can jump to any other configuration \( \sigma \in \Sigma^q \) in a single step, with probability determined solely by the marginal \( \mu_t(\sigma) \). There is no notion of continuity or neighborhood preserved by the dynamics. This contrasts sharply with reverse processes, where transitions are typically limited to configurations that differ by only one coordinate.

Thus, while the degenerate reverse process is mathematically valid and correctly reproduces the marginal distributions \( \mu_t \), it does not preserve the locality structure of the forward process. Its ability to transition freely between any two configurations in \( \Sigma^q \) , without regard for neighborhood structure, leads to a reverse kernel that is inherently non-local. In high-dimensional spaces, such non-local kernels operate over the entire \( \Sigma^q \times \Sigma^q \) transition space, making them exponentially more complex to represent, learn, or approximate. On the other hand, kernels for local update rules scale linearly and generalize more easily.

\section{Numerical Implementation Details}
\label{sec:implementation}

\subsection{Model Architecture}

We use the following architecture for each of the cases :

\begin{itemize}
\item Input block:
$\mathrm{Linear}(d_{\text{in}} \to h)\rightarrow
\mathrm{LayerNorm}(h)\rightarrow
\mathrm{SiLU}$,  
where $d_{\text{in}}$ is dependent on the denoising algorithm.
\item Hidden blocks (up to $D{-}1$ blocks, depending on depth $D\in\{1,\dots,5\}$):  
$\mathrm{Linear}(h \to h)\rightarrow
\mathrm{LayerNorm}(h)\rightarrow
\mathrm{SiLU}$.

\item Output layer:
$\mathrm{Linear}(h \to 2)$.
\end{itemize}

\subsection{Shared hyperparameter sweep}
For every dataset--denoising algorithm combination, we run a small hyperparameter optimization loop over the following parameter using the {\it hyperopt package} in Python:
\begin{itemize}
\item \textbf{Depth:} $D\in\{1,..,5\}$ 
\item \textbf{Width:} $h \in\{64,128,256,512\}$ 
\item \textbf{Noise parameter} (Only for NeurISE Diffusion): $\varepsilon \in (0,1)$ 
\item \textbf{Noising time horizon (Only for NeurISE Diffusion): $T \in [0,2,...10]$ }
\item \textbf{Learning rate: $ {\rm lr} \in (10^{-4}, 5 \times 10^{-2})$} with log uniform distribution.
\item \textbf{Weight decay: $ w \in (10^{-8}, 10^{-3}$} with log uniform distribution.
\item \textbf{Batch Size: $[64,128,256,512]$} 
\end{itemize}
\subsection{SEDD Implementation details}

In the implementation of SEDD \cite{lou2024discrete}, we introduced a final layer that enforced positivity in the output of the layer for the score approximation. While it is claimed that the loss function introduced in \cite{lou2024discrete} naturally forces the output of the network towards non-negativity, we did not observe this in our implementation, and in fact found that the training algorithm returned NaNs if the final layer was not appropriately augmented. Additionally, the inputs were required to be one-hot coded entirely in order for algorithm to show any significant learning. In contrast, for the Neurise Diffusion, GGM, D3PM and DFM only conditioning parameters were required to be one-hot coded. 

\newpage

\end{document}